\documentclass[11pt]{article}

\usepackage[T1]{fontenc}
\usepackage[utf8]{inputenc}

\usepackage{graphicx}
\usepackage{booktabs}
\usepackage{xcolor}
\usepackage{url}
\usepackage{cancel}
\definecolor{darkgreen}{rgb}{0.4,0.8,0}
\usepackage{todonotes}
\usepackage{xargs}

\usepackage{amsmath,amsthm,amssymb}
\usepackage[showonlyrefs]{mathtools}            

\numberwithin{equation}{section}

\newtheorem{theorem}{Theorem}[section]
\newtheorem{prop}[theorem]{Proposition}

\newtheorem{corollary}[theorem]{Corollary}
\theoremstyle{definition}

\usepackage{algorithm}
\usepackage{algpseudocode}

\usepackage[a4paper, margin=2.5cm]{geometry}

\usepackage{float}
\usepackage{subcaption}

\usepackage[round,authoryear]{natbib}
\definecolor{darkblue}{rgb}{0,0,0.6}
\usepackage[colorlinks=true,
            allcolors=darkblue]{hyperref}
\usepackage{cleveref}
\usepackage{chngcntr}

\newcommand{\R}{\mathbb R}

\newcommand{\LeftEqNo}{\let\veqno\@@leqno}

\newcommand{\N}{\ensuremath{\mathbb{N}}}

\newcommandx{\Vnorm}[2][1=V]{\| #2 \|_{#1}}
\newcommandx{\VnormEq}[2][1=V]{\left\| #2 \right\|_{#1}}
\newcommandx{\norm}[2][1=]{\ifthenelse{\equal{#1}{}}{\left\vert #2 \right\vert}{\left\vert #2 \right\vert^{#1}}}
\newcommandx{\normLigne}[2][1=]{\ifthenelse{\equal{#1}{}}{\Vert #2 \Vert}{\Vert #2\Vert^{#1}}}

\newcommandx\probaMarkovTilde[2][2=]
{\ifthenelse{\equal{#2}{}}{{\widetilde{\mathbb{P}}_{#1}}}{\widetilde{\mathbb{P}}_{#1}\left[ #2\right]}}

\def\ie{\textit{i.e.}}

\def\eqsp{\;}

\newcommandx{\weight}[2][2=n]{\omega_{#1,#2}^N}

\newcommandx\sequence[3][2=,3=]
{\ifthenelse{\equal{#3}{}}{\ensuremath{\{ #1_{#2}\}}}{\ensuremath{\{ #1_{#2}, \eqsp #2 \in #3 \}}}}
\newcommandx\sequenceD[3][2=,3=]
{\ifthenelse{\equal{#3}{}}{\ensuremath{\{ #1_{#2}\}}}{\ensuremath{( #1)_{ #2 \in #3} }}}

\newcommandx{\sequencen}[2][2=n\in\N]{\ensuremath{\{ #1_n, \eqsp #2 \}}}
\newcommandx\sequenceDouble[4][3=,4=]
{\ifthenelse{\equal{#3}{}}{\ensuremath{\{ (#1_{#3},#2_{#3}) \}}}{\ensuremath{\{  (#1_{#3},#2_{#3}), \eqsp #3 \in #4 \}}}}
\newcommandx{\sequencenDouble}[3][3=n\in\N]{\ensuremath{\{ (#1_{n},#2_{n}), \eqsp #3 \}}}

\newcommand{\opnorm}[1]{{\left\vert\kern-0.25ex\left\vert\kern-0.25ex\left\vert #1 
    \right\vert\kern-0.25ex\right\vert\kern-0.25ex\right\vert}}

\newcommandx{\CPE}[3][1=]{{\mathbb E}_{#1}\left[\left. #2 \, \middle \vert \, #3 \right. \right]} 

\newcommandx{\CPELigne}[3][1=]{{\mathbb E}_{#1}[\left. #2 \,  \vert \, #3 \right. ]} 

\newcommandx{\CPVar}[3][1=]{\mathrm{Var}^{#3}_{#1}\left\{ #2 \right\}}
\newcommand{\CPP}[3][]
{\ifthenelse{\equal{#1}{}}{{\mathbb P}\left(\left. #2 \, \right| #3 \right)}{{\mathbb P}_{#1}\left(\left. #2 \, \right | #3 \right)}}

\newcommandx{\osc}[2][1=]{\mathrm{osc}_{#1}(#2)}

\newcommand\coupling[2]{\Gamma(\mu,\nu)}

\renewcommand{\geq}{\geqslant}
\renewcommand{\leq}{\leqslant}

\newcommandx{\wasserstein}[3][1=\distance,3=]{\mathscr{W}_{#1}^{#3}\left(#2\right)}
\newcommandx{\wassersteinLigne}[3][1=\distance,3=]{\mathscr{W}_{#1}^{#3}(#2)}
\newcommandx{\wassersteinD}[1][1=\distance]{\mathscr{W}_{#1}}
\newcommandx{\wassersteinDLigne}[1][1=\distance]{\mathscr{W}_{#1}}

\def\I{\boldsymbol{\mathrm{I}}}

\def\bfo{\mathbf{o}}

\newcommandx{\Voi}[1][1=i]{\mathfrak{V}_{\bfo,#1}}
\newcommandx{\Vlyapc}[2][1=\bfo,2=i]{\mathfrak{V}_{#1,#2}}

\def\Wlyap{\mathfrak{W}}
\def\bfomega{\boldsymbol{\omega}}

\newcommandx{\Woi}[1][1=i]{\Wlyap_{\bfomega,#1}}
\newcommandx{\Wlyapc}[2][1=\bfomega,2=i]{\Wlyap_{#1,#2}}

\definecolor{darkblue}{rgb}{0,0,.4}

\providecommand{\State}{\STATE}

\newcommand{\E}{\mathbb{E}}

\newcommand{\simiid}{\overset{\mathrm{iid}}{\sim}}

\title{Iterative Tilting for Diffusion Fine-Tuning}
\author{
Jean Pachebat\textsuperscript{1} \quad
Giovanni Conforti\textsuperscript{2} \quad
Alain Durmus\textsuperscript{1} \quad
Yazid Janati\textsuperscript{3}
\\[0.8em]
{\normalsize \textsuperscript{1}CMAP, \'Ecole Polytechnique}\\
{\normalsize \textsuperscript{2}Universit\`a degli Studi di Padova}\\
{\normalsize \textsuperscript{3}Institute of Foundation Models}
}
\date{}

\begin{document}

\maketitle

\begin{abstract}
We introduce iterative tilting, a gradient-free method for fine-tuning diffusion models toward reward-tilted distributions. The method decomposes a large reward tilt $\exp(\lambda r)$ into $N$ sequential smaller tilts, each admitting a tractable score update via first-order Taylor expansion. This requires only forward evaluations of the reward function and avoids backpropagating through sampling chains. We validate on a two-dimensional Gaussian mixture with linear reward, where the exact tilted distribution is available in closed form.
\end{abstract}

\section{Introduction}

Diffusion models are now a standard tool for generative modelling in images, audio, and related domains \citep{sohl-dickstein2015,ho2020,song2021}.
They learn a time-dependent score (or denoiser) that inverts a simple noising process and produces new samples from noise.
Following their success, it has been proposed to fine-tune pretrained diffusion models, taking into account new data preferences and/or constraints by simply shifting the learnt weights accordingly.

To achieve this, ideas from language-model alignment have recently started to be applied to diffusion:
Either shifting the pre-trained weights using methods based on reward models \citep{xu2023b} or reinforcement learning policies \citep{black2024a,fan2023}, or adapting the architecture using lightweight adapters such as LoRA \citep{hu2021}.
In addition, diffusion-specific techniques like classifier-free guidance \citep{ho2022} or improved noise schedulers \citep{nichol2021a} provide complementary control mechanisms.

In this paper we introduce \textit{iterative tilting}, a gradient-free method for reward-based fine-tuning that shifts a diffusion model toward a reward-tilted density without differentiating through the reward function.
During the preparation of this work, a concurrent paper was submitted to ICLR \citep{anonymous2025tilt}. Their method is essentially the same as ours but formulated using the stochastic interpolants framework and developed in continuous time. In contrast, our approach is based on a diffusion framework and is derived in discrete time.

Section~\ref{sec:preliminaries} recalls the theoretical framework of Denoising Diffusion Models (DDM) and reviews existing reward fine-tuning formulations, Section~\ref{sec:outline} develops the Iterative Tilting method from a theoretical standpoint, and Section~\ref{sec:experiments} reports controlled experiments on a Gaussian synthetic fine-tuning task, before we conclude.

\section{Preliminaries on Denoising Diffusion Models and Fine-Tuning}
\label{sec:preliminaries}

\subsection{Diffusion models}
\label{subsec:diffusion_models}

Denoising diffusion models (DDMs) \citep{sohl-dickstein2015,song2019,ho2020} define a generative process for a data distribution $p_0$ by constructing a continuous probability path $(p_t)_{t\in[0,1]}$ of distributions between $p_0$ and a tractable reference $p_1 := \mathrm{N}(0, \mathrm{I}_d)$.

\subsubsection{Forward noising process}

The path is defined via the marginals $p_t = \mathrm{Law}(X_t)$, where
\begin{equation}
    X_t = \alpha_t X_0 + \sigma_t X_1, \qquad X_0 \sim p_0, \quad X_1 \sim \mathrm{N}(0, \mathrm{I}_d).
    \label{eq:forward_interpolation}
\end{equation}
Here $X_0$ and $X_1$ are independent, and $(\alpha_t)_{t\in[0,1]}$ and $(\sigma_t)_{t\in[0,1]}$ are deterministic schedules such that $\alpha_t$ is non-increasing, $\sigma_t$ is non-decreasing, with boundary conditions $(\alpha_0, \sigma_0) = (1, 0)$ and $(\alpha_1, \sigma_1) = (0, 1)$. The path $(p_t)_{t\in[0,1]}$ thus defines an interpolation that gradually transforms the clean data distribution $p_0$ into the Gaussian reference $p_1$.

From \eqref{eq:forward_interpolation}, the conditional distribution of $X_t$ given $X_0 = x_0$, denoted $q_{t|0}$, is Gaussian:
\begin{equation}
    q_{t|0}(x_t \mid x_0) = \mathrm{N}(x_t;\, \alpha_t x_0,\, \sigma_t^2 \mathrm{I}_d).
    \label{eq:forward_kernel}
\end{equation}

Regarding the choices of schedules $(\alpha_t, \sigma_t)$, two standard choices are the \emph{variance-preserving (VP)} schedule with $\alpha_t^2 + \sigma_t^2 = 1$ \citep{ho2020,dhariwal2021}, which ensures $\mathrm{Var}(X_t) = \mathrm{Var}(X_0)$ when $X_0$ has unit variance, and the \emph{linear (flow matching)} schedule with $(\alpha_t, \sigma_t) = (1-t, t)$ \citep{lipman2023,esser2024}, corresponding to a straight-line interpolation between data and noise.

\subsubsection{Reverse process and DDIM transitions}
\label{subsec:ddim}

To generate new samples, DDMs simulate a time-reversed Markov chain. Given a increasing sequence $(t_k)_{k=0}^K$ of time steps with $t_K = 1$ and $t_0 = 0$, reverse transitions are iteratively applied to map a sample from $p_{t_{k+1}}$ to one from $p_{t_k}$, progressively denoising until convergence to $p_0$.

\paragraph{DDIM framework.}
The DDIM framework \citep{song2021ddim} introduces a general family of reverse transitions. It defines a schedule $(\eta_t)_{t\in[0,1]}$ satisfying $\eta_t \leq \sigma_t$ for all $t \in [0,1]$, along with transition densities given for $s < t$ by
\begin{equation}
    p_{s|t}^\eta(x_s \mid x_t) = \E\left[ q_{s|0,1}^\eta(x_s \mid X_0, X_1) \,\middle|\, X_t = x_t \right],
    \label{eq:ddim_transition}
\end{equation}
where
\begin{equation}
    q_{s|0,1}^\eta(x_s \mid x_0, x_1) := \mathrm{N}\!\left(x_s;\, \alpha_s x_0 + \sqrt{\sigma_s^2 - \eta_s^2}\, x_1,\, \eta_s^2 \mathrm{I}_d\right)
    \label{eq:ddim_bridge}
\end{equation}
and the random variables $(X_0, X_t, X_1)$ are defined as in \eqref{eq:forward_interpolation}. By construction, this family satisfies the marginalisation property $p_s(x_s) = \int p_{s|t}^\eta(x_s \mid x_t)\, p_t(x_t)\, \mathrm{d}x_t$: by the law of total expectation, the right-hand side equals $\E_{(X_0, X_1)}[q_{s|0,1}^\eta(x_s \mid X_0, X_1)]$, and since $X_1 \sim \mathrm{N}(0, \mathrm{I}_d)$ is independent of $X_0$, sampling from the bridge \eqref{eq:ddim_bridge} yields $X_s = \alpha_s X_0 + \sqrt{\sigma_s^2 - \eta_s^2}\, X_1 + \eta_s \epsilon$ with $\epsilon \sim \mathrm{N}(0, \mathrm{I}_d)$; the sum $\sqrt{\sigma_s^2 - \eta_s^2}\, X_1 + \eta_s \epsilon$ has variance $\sigma_s^2 \mathrm{I}_d$, so $X_s \stackrel{d}{=} \alpha_s X_0 + \sigma_s \tilde{X}_1$ with $\tilde{X}_1 \sim \mathrm{N}(0, \mathrm{I}_d)$, matching the definition of $p_s$. This ensures that $(p_{t_k|t_{k+1}}^\eta)_{k=0}^{K-1}$ defines a consistent set of reverse transitions.

\paragraph{Denoiser.}
The transitions \eqref{eq:ddim_transition} are intractable and have to be approximated.
To this end, it has been proposed to replace $X_0$ and $X_1$ by their conditional expectations given $X_t$. Let $\hat{x}_0^\theta(\cdot, t)$ and $\hat{x}_1^\theta(\cdot, t)$ denote parametric estimators of
\begin{equation}
    \hat{x}_0(x_t, t) := \E[X_0 \mid X_t = x_t], \qquad \hat{x}_1(x_t, t) := \E[X_1 \mid X_t = x_t].
    \label{eq:denoiser_def}
\end{equation}
From the forward interpolation \eqref{eq:forward_interpolation}, these two quantities are related by $\hat{x}_1(x_t, t) = (x_t - \alpha_t \hat{x}_0(x_t, t))/\sigma_t$, yielding the parametric relationship
\begin{equation}
    \hat{x}_1^\theta(x_t, t) = \frac{x_t - \alpha_t \hat{x}_0^\theta(x_t, t)}{\sigma_t}, \qquad \hat{x}_0^\theta(x_t, t) = \frac{x_t - \sigma_t \hat{x}_1^\theta(x_t, t)}{\alpha_t}.
    \label{eq:denoiser_relationship}
\end{equation}
The parametric DDIM transition then becomes
\begin{equation}
    p_{t_k|t_{k+1}}^{\eta,\theta}(x_{t_k} \mid x_{t_{k+1}}) := q_{t_k|0,1}^\eta\!\left(x_{t_k} \mid \hat{x}_0^\theta(x_{t_{k+1}}, t_{k+1}),\, \hat{x}_1^\theta(x_{t_{k+1}}, t_{k+1})\right).
    \label{eq:ddim_parametric}
\end{equation}
For $k = 0$, the transition $p_{0|t_1}^{\eta,\theta}(\cdot \mid x_{t_1})$ is simply defined as the Dirac mass at $\hat{x}_0^\theta(x_{t_1}, t_1)$.

\paragraph{Sampling procedure.}
Given timesteps $(t_k)_{k=0}^K$ with $t_K = 1$ and $t_0 = 0$, DDIM sampling is summarised in Algorithm~\ref{alg:ddim}. Setting $\eta_t = 0$ yields deterministic sampling; setting $\eta_t = \sigma_t\sqrt{1 - \alpha_t^2/\alpha_s^2}$ (for the VP schedule) recovers the stochastic DDPM sampler \citep{ho2020}.

\begin{algorithm}[t]
    \caption{DDIM Sampling}
    \label{alg:ddim}
    \begin{algorithmic}[1]
        \Require Denoiser $\hat{x}_1^\theta$, timesteps $(t_k)_{k=0}^K$ with $t_K = 1$, $t_0 = 0$, variance schedule $(\eta_{t_k})_{k=0}^{K-1}$
        \State Sample $x_{t_K} \sim \mathrm{N}(0, \mathrm{I}_d)$
        \For{$k = K-1, \ldots, 1$}
            \State $\hat{x}_1 \leftarrow \hat{x}_1^\theta(x_{t_{k+1}}, t_{k+1})$ \Comment{Predict $X_1$}
            \State $\hat{x}_0 \leftarrow (x_{t_{k+1}} - \sigma_{t_{k+1}} \hat{x}_1)/\alpha_{t_{k+1}}$ \Comment{Estimate clean sample}
            \State $\mu_{t_k} \leftarrow \alpha_{t_k} \hat{x}_0 + \sqrt{\sigma_{t_k}^2 - \eta_{t_k}^2}\, \hat{x}_1$
            \State Sample $z \sim \mathrm{N}(0, \mathrm{I}_d)$
            \State $x_{t_k} \leftarrow \mu_{t_k} + \eta_{t_k}\, z$
        \EndFor
        \State $x_0 \leftarrow (x_{t_1} - \sigma_{t_1} \hat{x}_1^\theta(x_{t_1}, t_1))/\alpha_{t_1}$ \Comment{Final denoising step}
        \State \Return $x_0$
    \end{algorithmic}
\end{algorithm}

\subsubsection{Score function and training}
\label{subsec:score_training}

\paragraph{Score and denoiser relationship.}
The \emph{score} of the noised distribution is $\nabla \log p_t(\cdot)$. For the Gaussian forward kernel \eqref{eq:forward_kernel}, the score admits a closed-form expression in terms of the denoiser. Exchanging expectation and gradient under standard regularity assumptions:
\begin{equation}
    \nabla_{x_t} \log p_t(x_t) = \E\left[\nabla_{x_t} \log q_{t|0}(x_t \mid X_0) \,\middle|\, X_t = x_t\right] = -\frac{\hat{x}_1(x_t, t)}{\sigma_t}.
    \label{eq:score_from_denoiser}
\end{equation}
Thus, the denoiser $\hat{x}_1^\theta$ directly provides a score estimator $s_\theta(x_t, t) = -\hat{x}_1^\theta(x_t, t)/\sigma_t$, and conversely $\hat{x}_1^\theta(x_t, t) = -\sigma_t s_\theta(x_t, t)$.

\paragraph{Training objective.}
The denoiser can be trained by regressing either $X_0$ or $X_1$ from the noised sample $X_t = \alpha_t X_0 + \sigma_t X_1$. The $X_0$-prediction loss reads
\begin{equation}
    \mathrm{L}_{X_0}(\theta) = \int_0^1 \E_{X_0 \sim p_0, X_1 \sim \mathrm{N}(0, \mathrm{I}_d)} \left\| \hat{x}_0^\theta(\alpha_t X_0 + \sigma_t X_1, t) - X_0 \right\|^2 \mathrm{d}t,
    \label{eq:x0_loss}
\end{equation}
while the $X_1$-prediction loss is
\begin{equation}
    \mathrm{L}_{X_1}(\theta) = \int_0^1 \E_{X_0 \sim p_0, X_1 \sim \mathrm{N}(0, \mathrm{I}_d)} \left\| \hat{x}_1^\theta(\alpha_t X_0 + \sigma_t X_1, t) - X_1 \right\|^2 \mathrm{d}t.
    \label{eq:x1_loss}
\end{equation}
Since $\hat{x}_1^\theta = -\sigma_t s_\theta$ by \eqref{eq:score_from_denoiser}, the $X_1$-prediction loss is equivalent to denoising score matching \citep{hyvarinen2005,vincent2011}: minimising \eqref{eq:x1_loss} amounts to learning the score $\nabla \log p_t$. In practice, the integral is approximated via Monte Carlo: sample $t \sim \mathrm{Unif}[0,1]$, $x_0 \sim p_0$, $x_1 \sim \mathrm{N}(0, \mathrm{I}_d)$, form $x_t = \alpha_t x_0 + \sigma_t x_1$, and regress either $x_0$ or $x_1$.

\subsubsection{Connection to SDEs}
\label{subsec:sde_connection}

The interpolation \eqref{eq:forward_interpolation} can be viewed as discretising a continuous-time process. Given a \emph{noise schedule} $\beta(t) > 0$, consider the It\^o SDE
\begin{equation}
    \mathrm{d}X_t = -\tfrac{1}{2}\beta(t) X_t\, \mathrm{d}t + \sqrt{\beta(t)}\, \mathrm{d}W_t, \qquad X_0 \sim p_0,
    \label{eq:forward_sde}
\end{equation}
where $(W_t)_{t\geq 0}$ is a standard Brownian motion in $\R^d$. For the VP schedule, setting $\alpha_t = \exp(-(1/2)\int_0^t \beta(s)\,\mathrm{d}s)$ and $\sigma_t = \sqrt{1 - \alpha_t^2}$, the marginal $\mathrm{Law}(X_t)$ coincides with $p_t$ from \eqref{eq:forward_interpolation}. The time-reversed process satisfies \citep{anderson1982,song2021}
\begin{equation}
    \mathrm{d}\bar{X}_t = \left[-\tfrac{1}{2}\beta(t)\bar{X}_t + \beta(t) \nabla_{x} \log p_t(\bar{X}_t)\right] \mathrm{d}t + \sqrt{\beta(t)}\, \mathrm{d}\bar{W}_t,
    \label{eq:reverse_sde}
\end{equation}
where $\bar{X}_t := X_{1-t}$ and $\bar{W}_t$ is a standard Brownian motion. The deterministic \emph{probability-flow ODE} \citep{song2021} with identical marginals reads
\begin{equation}
    \mathrm{d}x_t = \left[-\tfrac{1}{2}\beta(t) x_t + \tfrac{1}{2}\beta(t) \nabla_{x} \log p_t(x_t)\right] \mathrm{d}t.
    \label{eq:probability_flow_ode}
\end{equation}
DDIM with $\eta_t = 0$ can be seen as a discretisation of \eqref{eq:probability_flow_ode}, while DDPM corresponds to \eqref{eq:reverse_sde}.

\subsection{Fine-tuning formulations}
\label{subsec:fine_tuning_formulations}

Given a positive reward function $r$ and a pretrained diffusion model sampling from the base distribution $p_0^{\mathrm{b}}$, we aim to fine-tune the model to sample from the \emph{tilted distribution}, defined for $x\in \R^d$:
\begin{equation}
    p_0^{\lambda r}(x) = Z_{\lambda r}^{-1} \exp\!\left(\lambda r(x)\right) p_0^{\mathrm{b}}(x),
    \label{eq:reward_tilt}
\end{equation}
where $\lambda > 0$ is the tilt strength and $Z_{\lambda r} = \int \exp(\lambda r(x))\, p_0^{\mathrm{b}}(x)\, \mathrm{d}x$ is the normalising constant. This tilted distribution emphasises regions with high reward. Several formulations have been developed to tackle this goal.

\paragraph{Direct reward alignment (DRaFT).}
DRaFT \citep{clark2024} treats the diffusion sampler as a differentiable computation graph.
For network parameters $\theta$, the reverse sampling process is a deterministic function of $\theta$ and the i.i.d.\ Gaussian noises $\xi=(Z_K,\ldots, Z_1)$, denoted $(\theta,\xi) \mapsto \psi_0(\theta,\xi)$, \ie, $\psi_0(\theta,\xi) = \hat{X}_0$ where $\hat{X}_0$ is the output of Algorithm~\ref{alg:ddim}.

DRaFT maximises the expected reward:
\begin{equation}
    J(\theta) = \E_\xi\left[r\!\left(\psi_0(\theta,\xi)\right)\right]
    \label{eq:draft_objective}
\end{equation}
by backpropagating the reward gradient through the entire chain of reverse steps:
\begin{equation}
    \nabla_\theta J(\theta) = \E_\xi\left[\left(\frac{\partial \psi_0(\theta,\xi)}{\partial \theta}\right)^\top \nabla_{x} r\!\left(\psi_0(\theta,\xi)\right)\right].
    \label{eq:draft_grad}
\end{equation}

Full backpropagation is memory-intensive. DRaFT-$K$ reduces cost by differentiating only through the final $K$ steps (typically $K \in \{1, 4\}$), detaching earlier gradients. DRaFT-LV further refines $K = 1$ with a control-variate correction for lower variance while remaining unbiased.

\paragraph{Stochastic optimal control and Adjoint Matching.}
An alternative approach casts fine-tuning as stochastic optimal control (SOC) \citep{uehara2024b,tang2024,domingo-enrich2024b}. Rather than backpropagating through sampling (as in DRaFT), SOC methods add a drift control $u_t$ to shift the generative process toward high-reward regions. For a controlled SDE $\mathrm{d}X_t = [b(X_t, t) + g(t) u_t(X_t)]\, \mathrm{d}t + g(t)\, \mathrm{d}W_t$, where $g(t)$ denotes the diffusion coefficient (distinct from the noise schedule $\sigma_t$ in \eqref{eq:forward_interpolation}), the objective balances control effort against terminal reward:
\begin{equation}
    \min_u\; \E\left[\tfrac{1}{2}\int_0^1 \|u_t(X_t)\|^2\, \mathrm{d}t - \lambda\, r(X_1)\right].
    \label{eq:soc_objective}
\end{equation}
By Girsanov's theorem, the quadratic control cost equals the KL divergence between controlled and base processes. However, naïvely solving \eqref{eq:soc_objective} yields $p^*(X_0, X_1) \propto p^{\mathrm{base}}(X_0, X_1) \exp(\lambda r(X_1) + V(X_0, 0))$, where $V$ is the value function. This does not marginalize to the tilted distribution \eqref{eq:reward_tilt} unless the initial noise $X_0$ and generated sample $X_1$ are independent.

\cite{domingo-enrich2024b} show that a \emph{memoryless} diffusion coefficient, one that ensures $X_0 \perp X_1$, removes the value function bias. Under this schedule, the optimal control is $u_t^\star(x) = g(t)^\top \tilde{a}_t(x)$, where the \emph{lean adjoint} $\tilde{a}_t$ satisfies the backward ODE
\begin{equation}
    \tfrac{\mathrm{d}}{\mathrm{d}t}\tilde{a}_t = -[\nabla_x b(x, t)]^\top \tilde{a}_t, \qquad \tilde{a}_1 = \lambda \nabla_x r(X_1).
    \label{eq:am_adjoint}
\end{equation}
Adjoint Matching learns a parametric control $u_\phi$ by minimizing:
\begin{equation}
    \mathrm{L}_{\mathrm{AM}}(\phi) = \E\left[\int_0^1 \|u_\phi(X_t, t) + g(t)^\top \tilde{a}_t(X_t)\|^2\, \mathrm{d}t\right].
    \label{eq:am_loss}
\end{equation}
After training, the learned control is absorbed into the drift, yielding a fine-tuned score or velocity field. Sampling can then use any diffusion coefficient, including deterministic sampling with $g(t) = 0$.

\section{Iterative Tilting Method for Fine-Tuning Diffusion Models}
\label{sec:outline}

The methods reviewed in Section~\ref{subsec:fine_tuning_formulations} share a common limitation: they require differentiating through the reward function $r$. In DRaFT, reward gradients appear explicitly in \eqref{eq:draft_grad} via backpropagation through the sampling chain. In stochastic optimal control, they appear in the adjoint terminal condition $a_0 = \lambda \nabla_x r(X_0)$. When $r$ is a large pretrained model (such as a vision transformer or human preference classifier trained on millions of examples), computing these gradients becomes prohibitively expensive, both in memory and compute. Moreover, many reward functions of practical interest are non-differentiable or defined only through black-box evaluations.

We introduce \emph{iterative tilting}, a method that targets the reward-tilted distribution \eqref{eq:reward_tilt} using only forward evaluations of $r$, without any gradient computation. The key idea is to construct a path of intermediate distributions $(p_0^k)_{k=0}^N$ connecting the base $p_0^{\mathrm{b}}$ to the tilted $p_0^{\lambda r}$, each representing a gentle reweighting of the previous one. At each step, we reweight samples from the current distribution toward higher rewards, gradually approaching the target along this path. This sequential approach trades off a single expensive gradient-based update for multiple cheaper gradient-free updates.

This section develops the method in continuous time, which makes the structure of the approximation explicit and connects naturally to the score-based and control perspectives introduced in Section~\ref{sec:preliminaries}.

\subsection{Score of the tilted distribution}

Suppose we have a data distribution $X_0 \sim p_0^{\mathrm{b}}$ (the base distribution) and wish to sample from the tilted distribution \eqref{eq:reward_tilt}. We use the forward noising kernel $q_{t|0}$ from \eqref{eq:forward_kernel}, whose score is available in closed form:
\begin{equation}
\nabla_{x_t} \log q_{t|0}(x_t \mid x_0) = \frac{\alpha_t x_0 - x_t}{\sigma_t^2}.
\label{eq:grad_q_closed}
\end{equation}

Noising the tilted distribution, the marginal at time $t$ is:
\begin{equation}
    p_t^{\lambda r}(x_t) = \int q_{t|0}(x_t \mid x_0)\, p_0^{\lambda r}(x_0)\, \mathrm{d}x_0.
    \label{eq:marginal_tilted}
\end{equation}

For any prior density $p_0^{\circ}$ (e.g., base $p_0^{\mathrm{b}}$ or tilted $p_0^{\lambda r}$) and the fixed forward noising kernel $q_{t|0}(x_t \mid x_0)$, we write $p_{0|t}^{\circ}(x_0 \mid x_t)$ for the posterior (inference) density of the clean sample $X_0$ given its noised version $X_t = x_t$ at time $t$:
\begin{equation}
    p_{0|t}^{\circ}(x_0 \mid x_t) \propto p_0^{\circ}(x_0)\, q_{t|0}(x_t \mid x_0).
\end{equation}

\begin{prop}[Score of the tilted distribution]
\label{prop:score_tilted}
Assume that $r$ has at most linear growth, i.e., $\forall x\in\R^d:\;|r(x)| \leq C(1 + \|x\|)$ for some $C > 0$, and that $q_{t|0}$ is the Gaussian forward kernel \eqref{eq:forward_kernel}. Then for all $t \in (0,1]$ and $x_t \in \R^d$,
\begin{equation}
    \nabla_{x_t} \log p_t^{\lambda r}(x_t) = \frac{\E_{X_0 \sim p_{0|t}^{\mathrm{b}}(\cdot \mid x_t)}\!\left[\nabla_{x_t} \log q_{t|0}(x_t \mid X_0)\, \exp(\lambda r(X_0))\right]}{\E_{X_0 \sim p_{0|t}^{\mathrm{b}}(\cdot \mid x_t)}\!\left[\exp(\lambda r(X_0))\right]}.
    \label{eq:score_tilted_final}
\end{equation}
\end{prop}

\begin{proof}
Exchanging differentiation and integration in \eqref{eq:marginal_tilted} (justified by dominated convergence under the stated assumptions) yields
\begin{equation}
    \nabla_{x_t} p_t^{\lambda r}(x_t) = \int p_0^{\lambda r}(x_0)\, \nabla_{x_t} q_{t|0}(x_t \mid x_0)\, \mathrm{d}x_0.
\end{equation}
Dividing by $p_t^{\lambda r}(x_t)$ and using $\nabla_{x_t} p_t^{\lambda r}(x_t) = p_t^{\lambda r}(x_t) \nabla_{x_t} \log p_t^{\lambda r}(x_t)$, we obtain Fisher's identity:
\begin{equation}
    \nabla_{x_t} \log p_t^{\lambda r}(x_t) = \E_{X_0 \sim p_{0|t}^{\lambda r}(\cdot \mid x_t)}\!\left[\nabla_{x_t} \log q_{t|0}(x_t \mid X_0)\right].
    \label{eq:fisher_identity}
\end{equation}
The key observation is that the tilted posterior can be written as a reweighting of the base posterior:
\begin{align}
  p_{0|t}^{\lambda r}(x_0 \mid x_t) = \frac{p_0^{\lambda r}(x_0)\, q_{t|0}(x_t \mid x_0)}{p_t^{\lambda r}(x_t)} &= \frac{Z_{\lambda r}^{-1} \exp(\lambda r(x_0))\, p_0^{\mathrm{b}}(x_0)\, q_{t|0}(x_t \mid x_0)}{\int p_0^{\lambda r}(x_0')\, q_{t|0}(x_t \mid x_0')\, \mathrm{d}x_0'} \nonumber\\
    &= \frac{\exp(\lambda r(x_0))\, p_{0|t}^{\mathrm{b}}(x_0 \mid x_t)}{\E_{X_0' \sim p_{0|t}^{\mathrm{b}}(\cdot \mid x_t)}\!\left[\exp(\lambda r(X_0'))\right]},
    \label{eq:posterior_tilted}
\end{align}
where in the last equality, the marginal $p_t^{\mathrm{b}}(x_t)$ cancels out. The tilted posterior $p_{0|t}^{\lambda r}(\cdot \mid x_t)$ is thus an exponential tilt of the base posterior $p_{0|t}^{\mathrm{b}}(\cdot \mid x_t)$. Substituting \eqref{eq:posterior_tilted} into \eqref{eq:fisher_identity} gives \eqref{eq:score_tilted_final}.
\end{proof}

\paragraph{Iterative tilts.}
To efficiently approximate the score of the tilted distribution, we decompose the full tilt into a sequence of smaller tilts. Fix $N \in \N^*$ (typically large) and define, for $k = 0, \ldots, N$,
\begin{equation}
    p_0^{k}(x) := Z_{k}^{-1} \exp\!\left(\tfrac{k}{N} \lambda r(x)\right) p_0^{\mathrm{b}}(x),
    \label{eq:iterative_tilted_density}
\end{equation}
where $Z_k$ is the normalising constant. We have $p_0^{0} = p_0^{\mathrm{b}}$ and $p_0^{N} = p_0^{\lambda r}$. The key observation is that consecutive distributions are related by a small tilt: $p_0^{k}(x) \propto \exp(\lambda r(x)/N)\, p_0^{k-1}(x)$.

\begin{corollary}[First-order score approximation]
\label{cor:score_approx}
Under the assumptions of Proposition~\ref{prop:score_tilted}, for all $x_t \in \R^d$ and $t \in (0,1]$, define the posterior $\pi_{k-1}(\cdot) := p_{0|t}^{k-1}(\cdot \mid x_t)$ and the covariance
\begin{equation}
    C_{k-1}(x_t, t) := \mathrm{Cov}_{X_0 \sim \pi_{k-1}}\!\left(\nabla_{x_t} \log q_{t|0}(x_t \mid X_0),\, r(X_0)\right).
\end{equation}
Then, for large $N$,
\begin{equation}
    \nabla_{x_t} \log p_t^{k}(x_t) = \nabla_{x_t} \log p_t^{k-1}(x_t) + \frac{\lambda}{N}\, C_{k-1}(x_t, t) + O\!\left(\frac{1}{N^2}\right).
    \label{eq:score_iterative_tilted_approx_final}
\end{equation}
\end{corollary}

\begin{proof}
Applying Proposition~\ref{prop:score_tilted} with base $p_0^{k-1}$ and tilt strength $\lambda/N$:
\begin{equation}
    \nabla_{x_t} \log p_t^{k}(x_t) = \frac{\E_{\pi_{k-1}}\!\left[g(X_0)\, \exp(\tfrac{\lambda}{N} r(X_0))\right]}{\E_{\pi_{k-1}}\!\left[\exp(\tfrac{\lambda}{N} r(X_0))\right]},
    \label{eq:score_iterative_tilted}
\end{equation}
where $g(X_0) := \nabla_{x_t} \log q_{t|0}(x_t \mid X_0)$. Taylor-expanding $\exp(\lambda r / N) = 1 + \lambda r / N + O(N^{-2})$ in both numerator and denominator:
\begin{align*}
    \text{Numerator} &= \E_{\pi_{k-1}}[g] + \tfrac{\lambda}{N}\, \E_{\pi_{k-1}}[g\, r] + O(N^{-2}), \\
    \text{Denominator} &= 1 + \tfrac{\lambda}{N}\, \E_{\pi_{k-1}}[r] + O(N^{-2}).
\end{align*}
Dividing and keeping first-order terms yields $\E_{\pi_{k-1}}[g] + (\lambda / N)\, C_{k-1}(x_t, t) + O(N^{-2})$. By Fisher's identity, $\E_{\pi_{k-1}}[g] = \nabla_{x_t} \log p_t^{k-1}(x_t)$.
\end{proof}

\subsection{Neural approximation and algorithm}

\Cref{cor:score_approx} motivates a recursive algorithm for estimating the sequence of score functions
$\{\nabla_{x_t} \log p_t^{k}\}_{k=1}^{N}$ starting from $k=0$, for
which $p_t^{0} = p^{\mathrm{b}}_t$.  Suppose that at iteration
$k$, we have access to an approximation for
$\nabla_{x_t} \log p_t^{k}$, denoted by $s_{\theta^{k}}^{k}$. The resulting generative models associated to $s_{\theta^{k}}^{k}$ is called the teacher at step $k$.
\Cref{cor:score_approx} provides a natural procedure to obtain an approximation for $\nabla_{x_t} \log p_t^{k+1}$: we  train a new network $s_{\theta}^{k+1}$ with weights $\theta$, so that it matches the target score from \eqref{eq:score_iterative_tilted_approx_final}. The updated parameter $\theta^{k+1}$ is thus obtained as an approximate minimizer of the loss
\begin{equation}
    \mathrm{L}_{\mathrm{IT}}(\theta; k+1) = \int_0^1 \int_{\mathcal{X}} w(t) \left\| s_\theta(x_t, t) - \left(s_{\theta^{k}}^{k}(x_t, t) + \frac{\lambda}{N}\, C_{k}(x_t, t)\right)\right\|^2 p_t^{k}(x_t)\, \mathrm{d}x_t\, \mathrm{d}t \eqsp.
    \label{eq:iterative_tilting_loss}
  \end{equation}

  In practice, we make three simplifications. First, we use a Monte Carlo estimate of the loss using a
 teacher-generated dataset $\{X_0^i\}_{i+1}^N$ at the beginning of each step $k$. Second, for $X_t^i \sim q_{t|0}(\cdot | X_0^i)$,
  the covariance $C_{k}(X_t^i, t)$ is approximated by a single-sample estimate $g^i = r(X_0^i)( \nabla_{x_t} \log q_{t|0}(X_t^i \mid X_0^i)- s_{\theta^{k}}(X_t^i,t))$, where $\nabla_{x_t} \log q_{t|0}$ is computed via \eqref{eq:grad_q_closed}. Finally, we multiply by $\sigma_t^2$ to stabilize training across noise levels. This leads to the following empirical loss for a single sample time $t$ and pair $(X_0^i,X_1^i)$: setting $\delta =\lambda/N$,
\begin{equation}
    \mathrm{L}_{\mathrm{prac}}(\theta; k+1) = \sum_{i=1}^N \sigma_t^2 \left\| s_\theta^{k+1}(X_t^i, t) - \sigma_t^2 \left(s_{\theta^{k}}^k(X_t^i, t) + \delta\, r(X_0^i)\left(g^i - s_{\theta^{k}}^k(X_t^i, t)\right)\right)\right\|^2 \eqsp,
    \label{eq:it_practical_loss}
\end{equation}

Finally, since we  expect that $\nabla_{x_t} \log p_t^{k}$ to be close to $\nabla_{x_t} \log p_t^{k+1}$, the score network at step $k+1$, $s_{\theta}^{k+1}$ is initialized at $s_{\theta^{k}}^{k}$.
The iterative tilting procedure is summarised in Algorithm~\ref{alg:iterative_tilting}.
\begin{algorithm}
    \caption{Iterative Tilting}
    \label{alg:iterative_tilting}
    \begin{algorithmic}[1]
        \Require Pretrained model $s_{\theta^0}$, number of tilts $N$, tilt strength $\lambda$, samples per tilt $S$, reverse steps $n$, batch size $B$, epochs $E$, schedule $(\alpha_t, \sigma_t)$
        \State Set step size $\delta \leftarrow \lambda/N$
        \For{$k = 1$ to $N$}
            \State Generate dataset $\mathrm{D}_k$ of $S$ samples: $x_0 \leftarrow \mathrm{DDIM}(s_{\theta^{k}}, n)$
            \State Initialise $\theta \leftarrow \theta^{k}$ and freeze teacher $s_{\theta^{k}}^k$
            \For{epoch $= 1$ to $E$}
            \State Sample mini-batch $\{X_0^{(i)}\}_{i=1}^B \subset \mathrm{D}_k$, $\{t_i\}_{i=1}^B \simiid \mathrm{Unif}[0,1]$, $\{Z_i\}_{i=1}^B \simiid \mathrm{N}(0, \mathrm{I}_d)$
                        \For{ $i= 1$ to $B$}
                \State Noise: $x_{t_i}^{(i)} = \alpha_{t_i} X_0^{(i)} + \sigma_{t_i} Z_i$
                \State Compute: $g_i = (\alpha_{t_i} X_0^{(i)} - X_{t_i}^{(i)})/\sigma_{t_i}^2$
                \State Compute: $s_i^{\mathrm{old}} = s_{\theta^{k}}^k(X_{t_i}^{(i)}, t_i)$
                \State Target: $\mathrm{target}_i = s_i^{\mathrm{old}} + \delta\, r(X_0^{(i)})(g_i - s_i^{\mathrm{old}})$
                \EndFor
                \State Update $\theta$ by one step of gradient descent on
                \[
                     \frac{1}{B} \sum_{i=1}^B \left\| \sigma_{t_i}^2 s_\theta(X_{t_i}^{(i)}, t_i) - \sigma_{t_i}^2 \mathrm{target}_i \right\|^2
                \]
            \EndFor
            \State Set $\theta^k \leftarrow \theta$
        \EndFor
        \State \Return $s_{\theta^N}$
    \end{algorithmic}
\end{algorithm}

\section{Experiments}
\label{sec:experiments}

In this section, we validate the iterative tilting method on a two-dimensional Gaussian mixture with a reward function. This setup has a key advantage: tilting a Gaussian (or mixture of Gaussians) with a quadratic reward (including linear as a special case) yields another Gaussian (or mixture), for which we can compute the exact tilted density and score analytically.

\subsection{Gaussian data with quadratic rewards}
\label{subsec:gaussian_potential}

\paragraph{One Gaussian.}
Suppose the base density is Gaussian:
\begin{equation}
    p^{\mathrm{b}}_0(x) = \frac{1}{(2\pi)^{d/2}|\Sigma|^{1/2}} \exp\left\{ -\tfrac{1}{2}(x-\mu)^{T}\Sigma^{-1}(x-\mu)\right\} = \mathrm{N}(x;\mu,\Sigma).
\end{equation}
and the reward function $r$ is quadratic, i.e.:
\begin{equation}
    r(x) = \tfrac{1}{2} x^{T} A x + b^{T}x + c.
\end{equation}
Then the product $\exp(r(x))p^{\mathrm{b}}_0(x)$ writes:
\begin{equation}
    \frac{1}{(2\pi)^{d/2}|\Sigma|^{1/2}} \exp\left\{ - \tfrac{1}{2}(x-\mu)^{T}\Sigma^{-1}(x-\mu) + \left(\tfrac{1}{2}x^{T}Ax + b^{T}x + c\right)\right\}.
\end{equation}
Focusing on the term in the exponent:
\begin{align}
    &-\tfrac{1}{2}(x-\mu)^{T}\Sigma^{-1}(x-\mu) + \left(\tfrac{1}{2}x^{T}Ax + b^{T}x + c\right) \nonumber \\
    &= -\tfrac{1}{2}\left[x^{T}(\Sigma^{-1}-A)x - 2(\Sigma^{-1}\mu + b)^{T}x\right] + c - \tfrac{1}{2}\mu^{T}\Sigma^{-1}\mu.
\end{align}
With \(\Sigma' = (\Sigma^{-1}-A)^{-1}\) and \(\mu' = \Sigma'(\Sigma^{-1}\mu + b)\) (we assume \(\Sigma^{-1}-A\) is positive definite so that the tilt is normalizable), the exponential can be rewritten as
\begin{equation}
    \exp(r(x))p^{\mathrm{b}}_0(x) = \left(\frac{|\Sigma'|}{|\Sigma|}\right)^{\!1/2}\exp\left\{\tfrac{1}{2}\big[\mu'^{T}\Sigma'^{-1}\mu' - \mu^{T}\Sigma^{-1}\mu\big] + c\right\}\, \mathrm{N}(x;\mu',\Sigma').
\end{equation}
Consequently,
\begin{equation}
    \int \exp(r(x))p^{\mathrm{b}}_0(x)\,dx = \left(\frac{|\Sigma'|}{|\Sigma|}\right)^{\!1/2}\exp\left\{\tfrac{1}{2}\big[\mu'^{T}\Sigma'^{-1}\mu' - \mu^{T}\Sigma^{-1}\mu\big] + c\right\}.
\end{equation}

\paragraph{Mixture of Gaussians.}
Suppose that we now have a mixture of Gaussians with $(\omega_{1},\ldots,\omega_{K}) \in \Delta_{K-1}$ where
\begin{equation}
    \Delta_{K-1} = \{(\omega_{1},\ldots,\omega_{K}) \in \mathbb{R}^{K}_{+}: \sum_{i\in[K]} \omega_{i}=1\}
\end{equation}
is the simplex and $(\mu_{1},\Sigma_{1}),\ldots,(\mu_{K},\Sigma_{K})$ corresponding means and covariances.
\begin{equation}
    p^{\mathrm{b}}_0(x) = \sum_{i\in[K]} \omega_{i}\,\mathrm{N}(x;\mu_{i},\Sigma_{i}).
\end{equation}
Then, we have:
\begin{align}
    \exp(r(x))p^{\mathrm{b}}_0(x) &= \sum_{i\in[K]} \omega_{i} \, \exp(r(x))\,\mathrm{N}(x;\mu_{i},\Sigma_{i}) \nonumber \\
    &= \sum_{i\in[K]} \omega'_{i}\,\mathrm{N}(x;\mu'_{i},\Sigma'_{i}),
\end{align}
where the parameters of the tilted Gaussians are

\begin{equation}
\left\{
\begin{aligned}
\Sigma'_{i} &= (\Sigma_{i}^{-1}-A)^{-1},\\
\mu'_{i} &= \Sigma'_{i}(\Sigma_{i}^{-1}\mu_{i} + b),\\
\omega'_{i} &= \omega_{i}\left(\frac{|\Sigma'_{i}|}{|\Sigma_{i}|}\right)^{\!1/2}\exp\left\{\tfrac{1}{2}\big[\mu'_{i}{}^{T}\Sigma'^{-1}_{i}\mu'_{i} - \mu_{i}^{T}\Sigma_{i}^{-1}\mu_{i}\big] + c\right\}.
\end{aligned}
\right.
\end{equation}

We consider the following mixture of two Gaussians in dimension $d=2$ as base distribution:

\begin{equation}
    \left\{
    \begin{array}{l}
        K=2 \\
        \mu_1 = (-2, 0);\, \mu_2 = (2, 0)\\
        \Sigma_1 = \Sigma_2 = \sigma^2 \I;\, \sigma^2 = 1 / 2   \\
        \omega_1 = \omega_2 = 1 / 2
    \end{array}
    \right.
\end{equation}

and we consider the linear reward $r(x) = b^Tx$ with $b = (0, 4)$ (corresponding to $A = 0$, $c=0$ in the general quadratic form).
In this case, the resulting distribution is a mixture of two Gaussians with parameters given by Section~\ref{subsec:gaussian_potential}:
\begin{equation}
    \left\{
    \begin{array}{l}
        K'=2 \\
        \mu'_1 = (-2, 2);\, \mu'_2 = (2, 2)\\
        \Sigma'_1 = \Sigma'_2 = \sigma^2\I;\,\sigma^2 = 1 /2   \\

        \omega'_1 = \omega'_2 = 1 / 2.
    \end{array}
    \right.
\end{equation}

\subsection{Implementation and results}
\paragraph{Experimental setup.}
We first train a diffusion model on the base distribution using the denoising score matching objective \eqref{eq:x1_loss} with the cosine noise schedule of \citep{nichol2021a}.
The architecture of the base model is a MLP model with depth $5$ of width $256$.
We embed the time $t$ with Fourier embeding \citep{tancik2020, song2021} of dimension $128$.
We train the model on a dataset of $30,000$ samples from the base distribution for $200$ epochs with early stopping (patience $50$).
The base model achieves a test negative log-likelihood of $3.278$ nats (computed on $5,000$ data points) and a mean squared error on the mean of $4.3\times 10^{-3}$.

We then fine-tune the model using the iterative tilting method (Algorithm~\ref{alg:iterative_tilting}) with $N\in\left\{20, 50, 100, 200\right\}$ tilts, $S=1,000$ samples per tilt, $J=200$ reverse steps with stochastic DDIM ($\eta=1.0$), batch size $B=64$ and $E=100$ epochs per tilt.

\paragraph{Results.}
We provide in Figure~\ref{fig:2d_gaussian_mixture_samples} samples obtained with the fine-tuned model at different stages of the tilting process for $N\in\{20, 50, 100, 200\}$.

\begin{figure}
    \centering
    \includegraphics[width=0.48\textwidth]{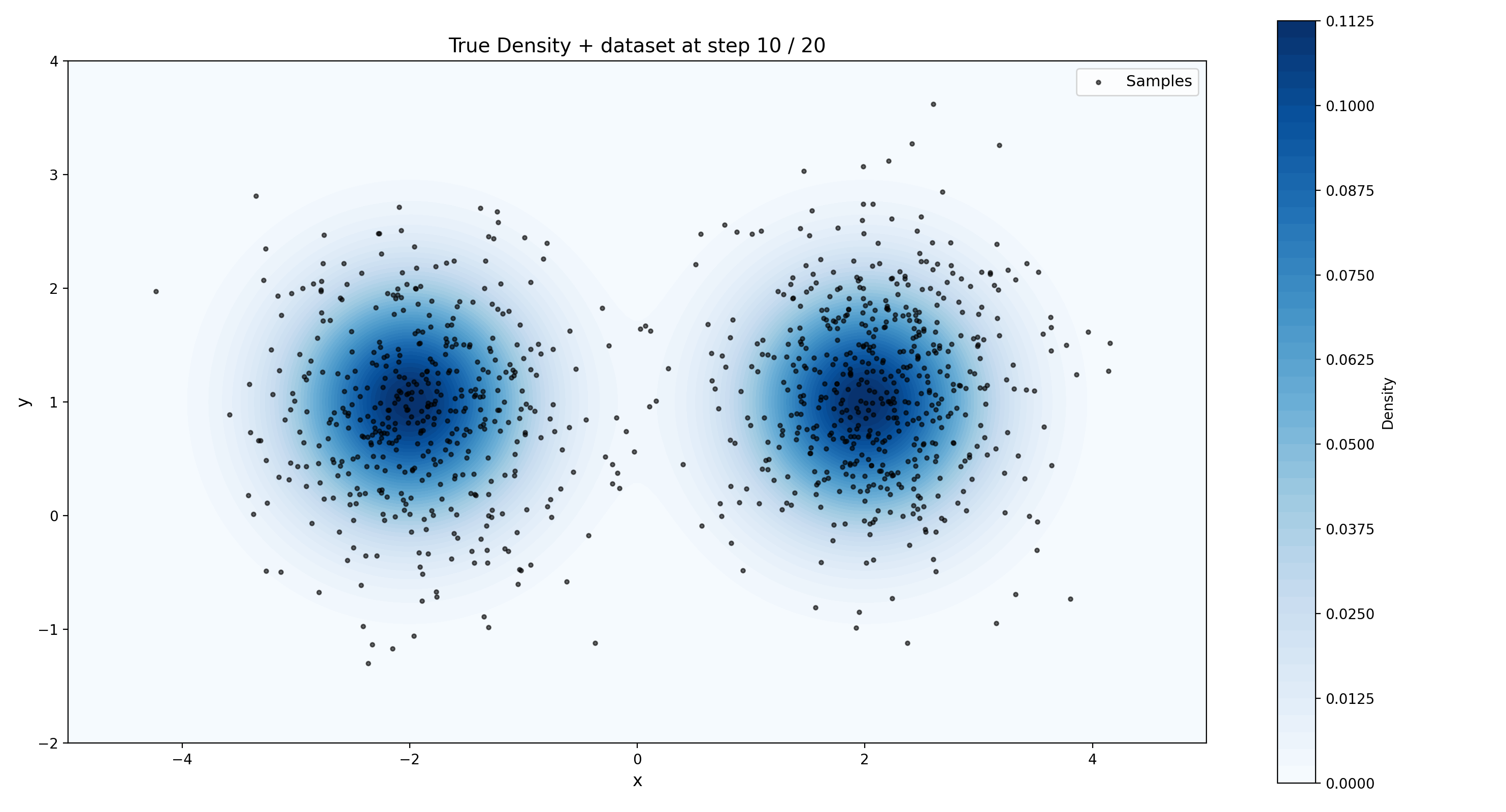}\hfill
    \includegraphics[width=0.48\textwidth]{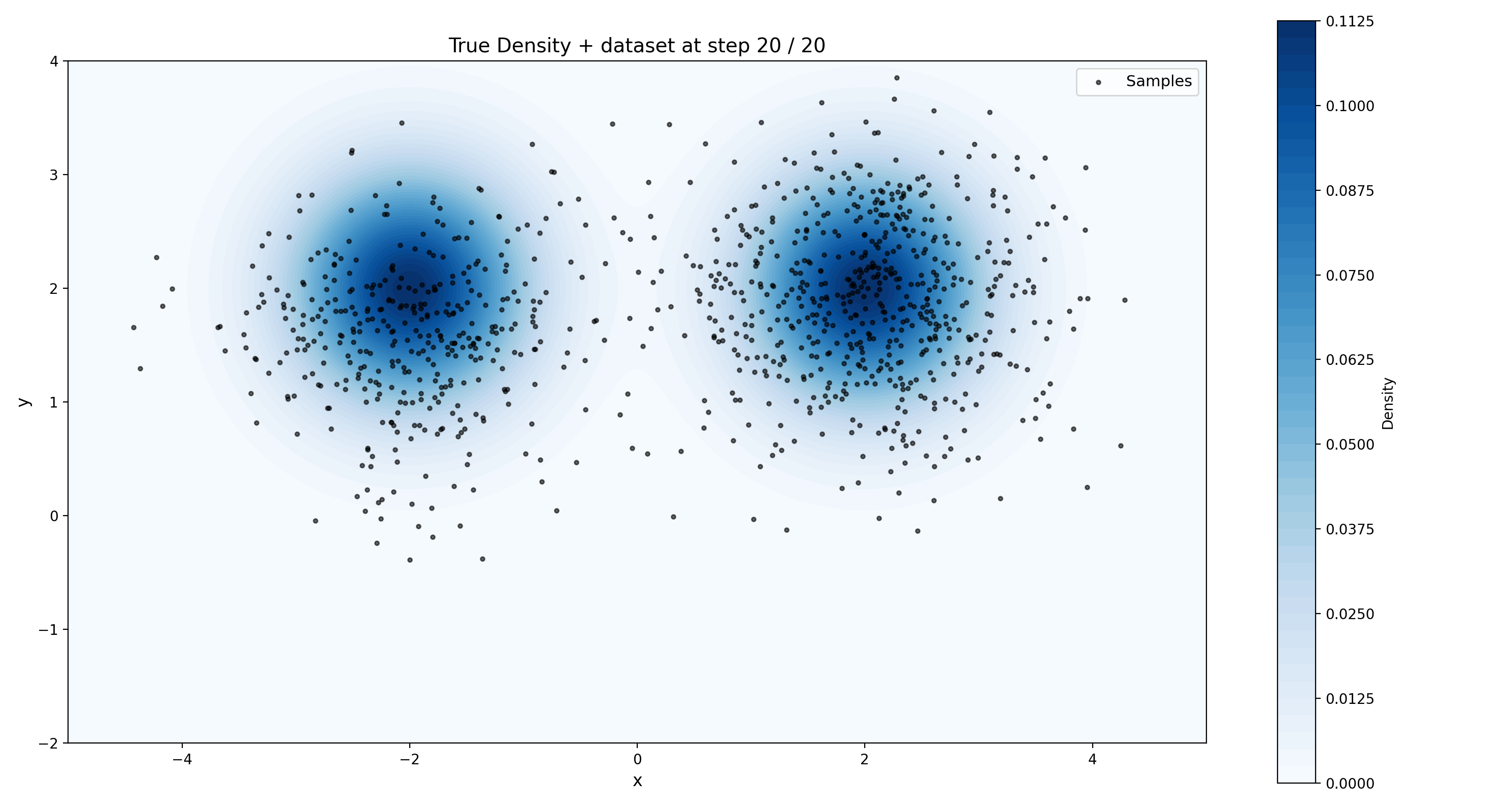}\\[0.6em]
    \includegraphics[width=0.48\textwidth]{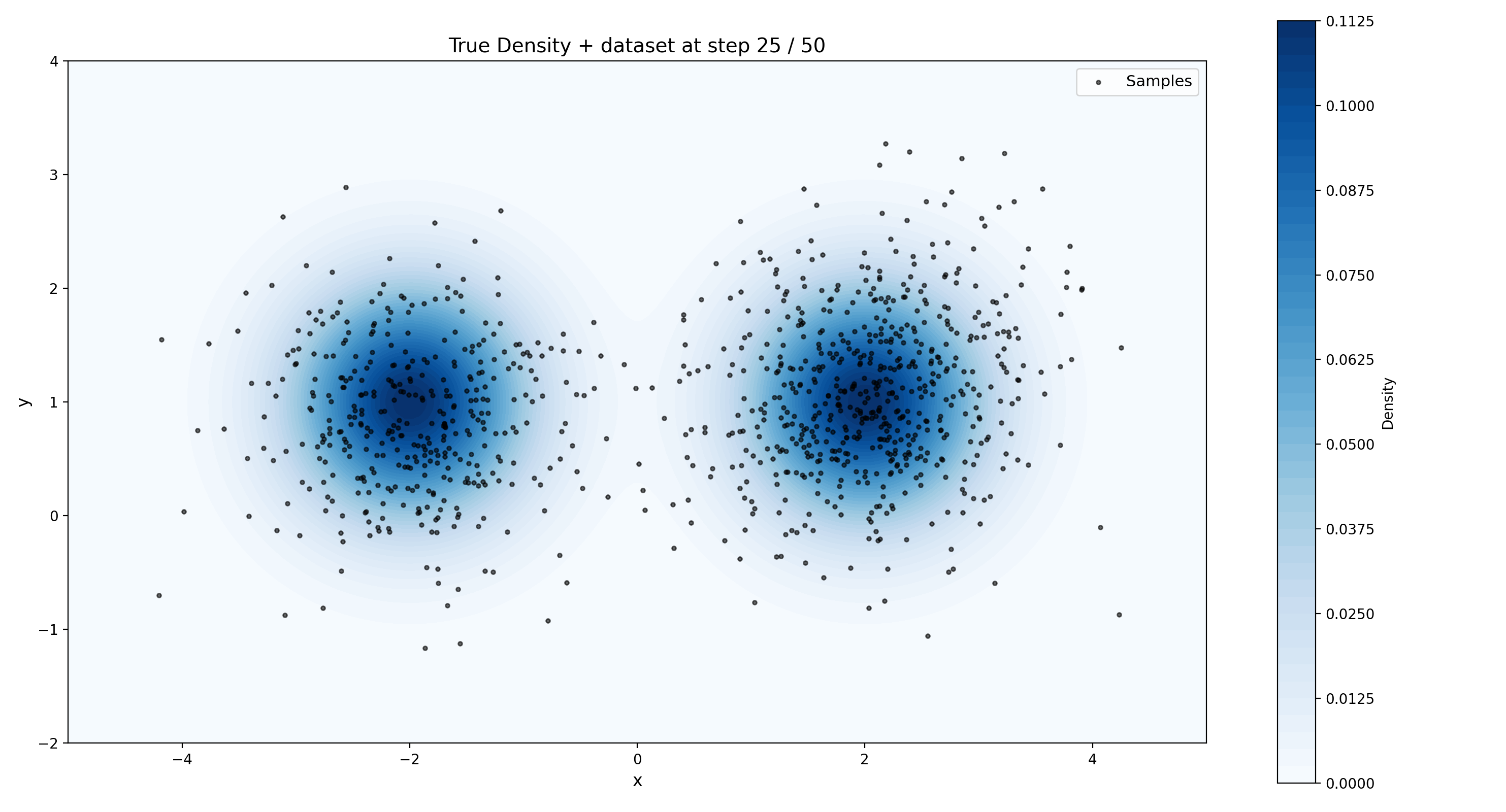}\hfill
    \includegraphics[width=0.48\textwidth]{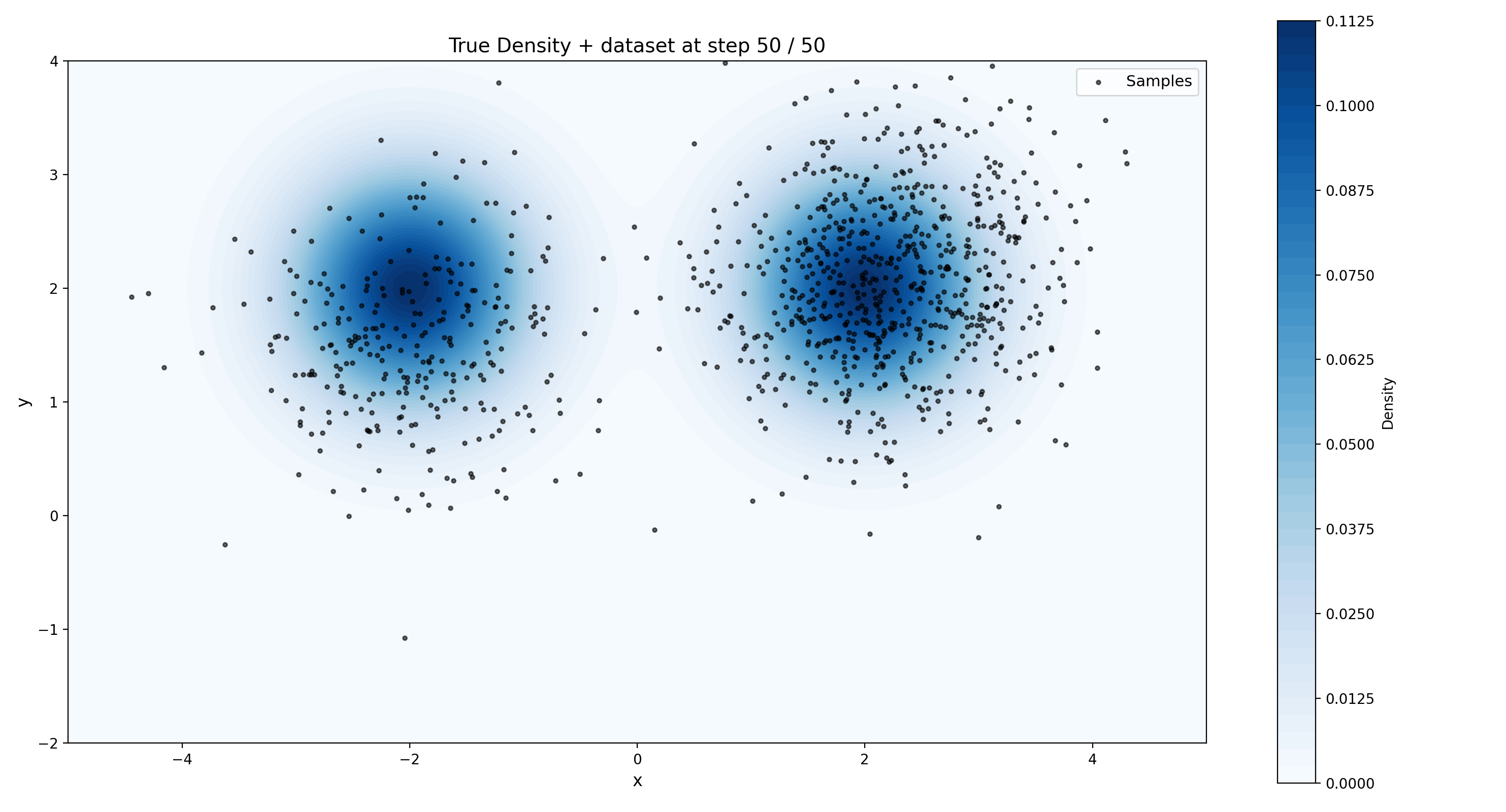}\\[0.6em]
    \includegraphics[width=0.48\textwidth]{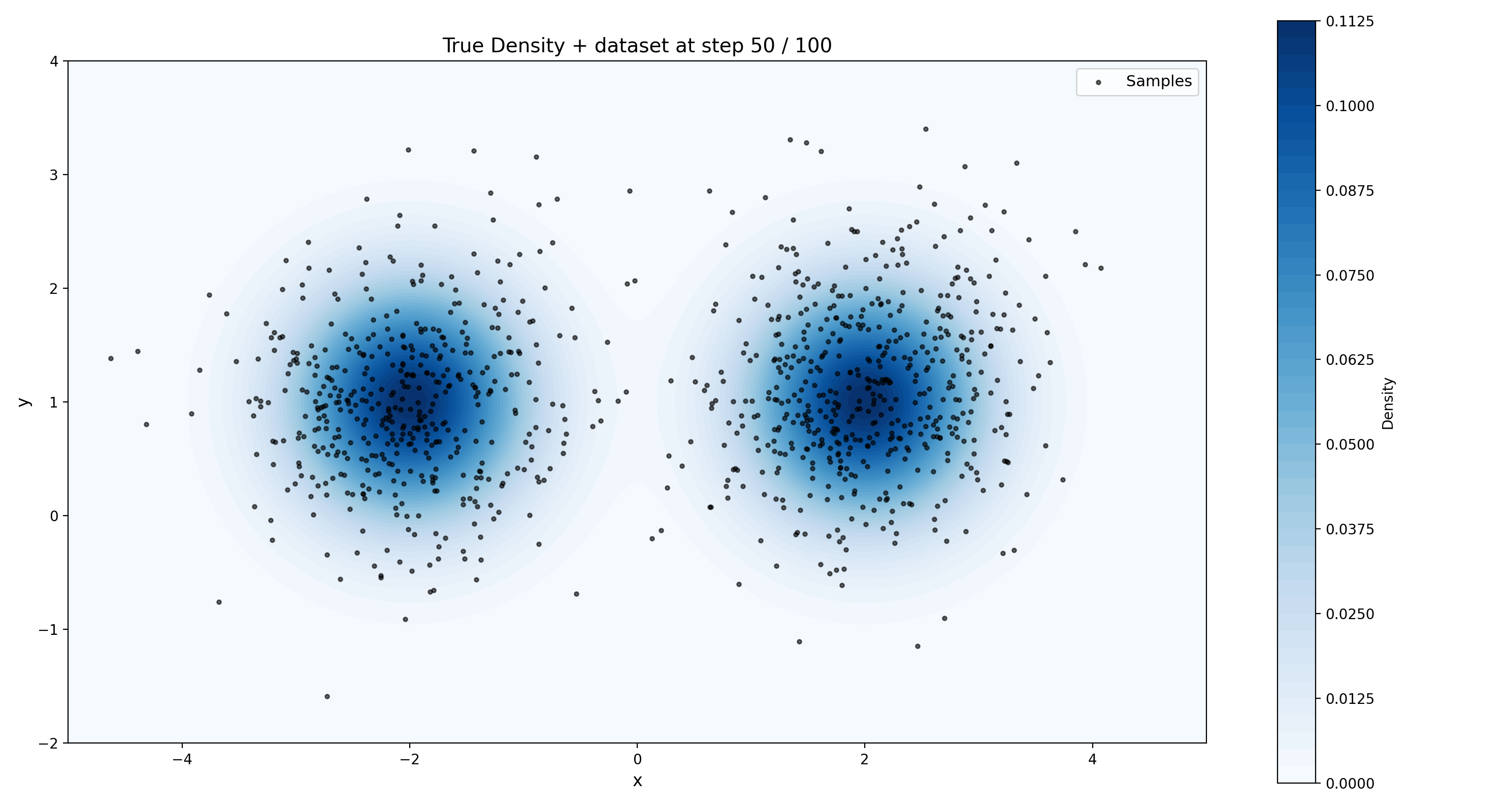}\hfill
    \includegraphics[width=0.48\textwidth]{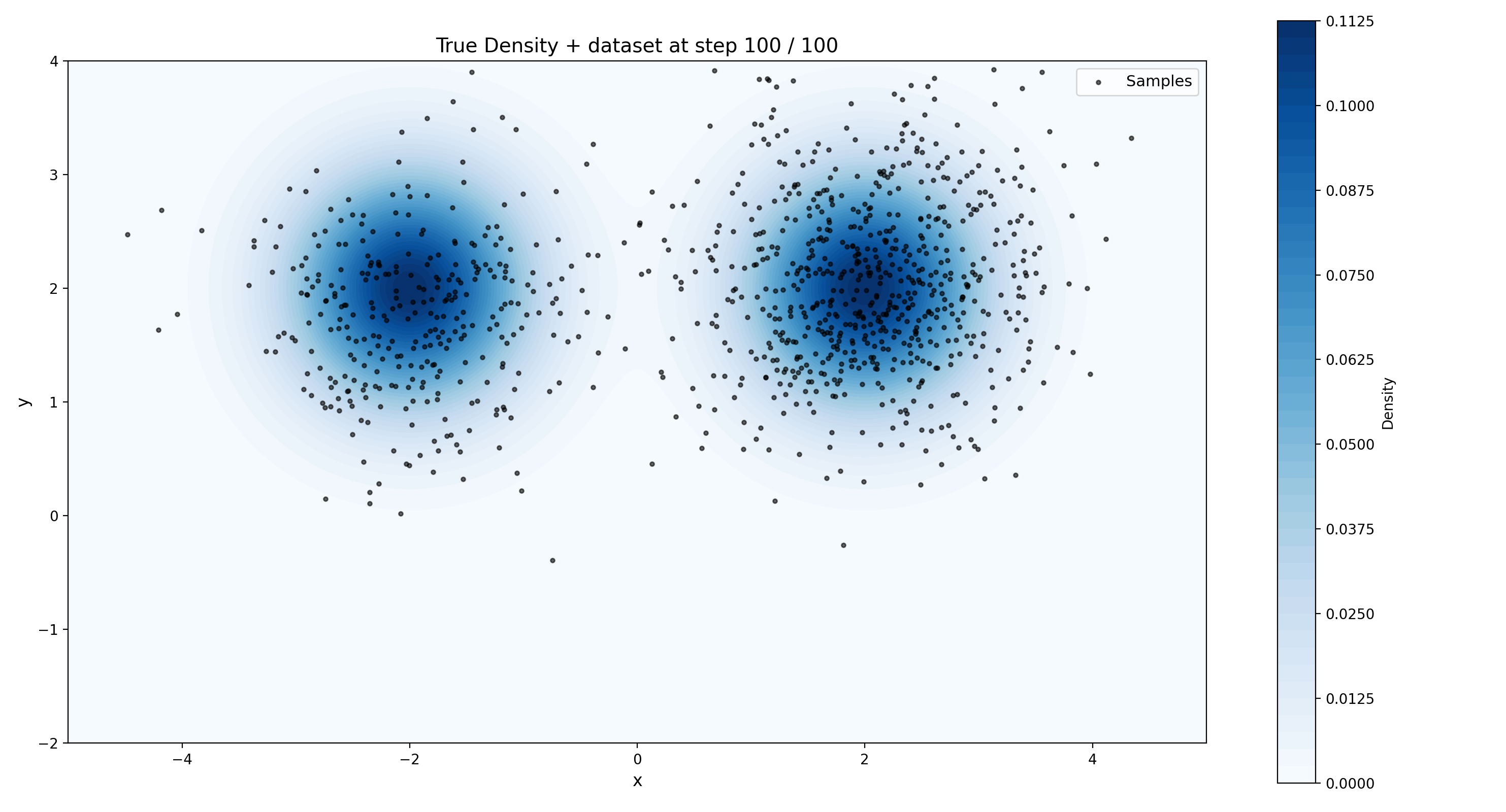}\\[0.6em]
    \includegraphics[width=0.48\textwidth]{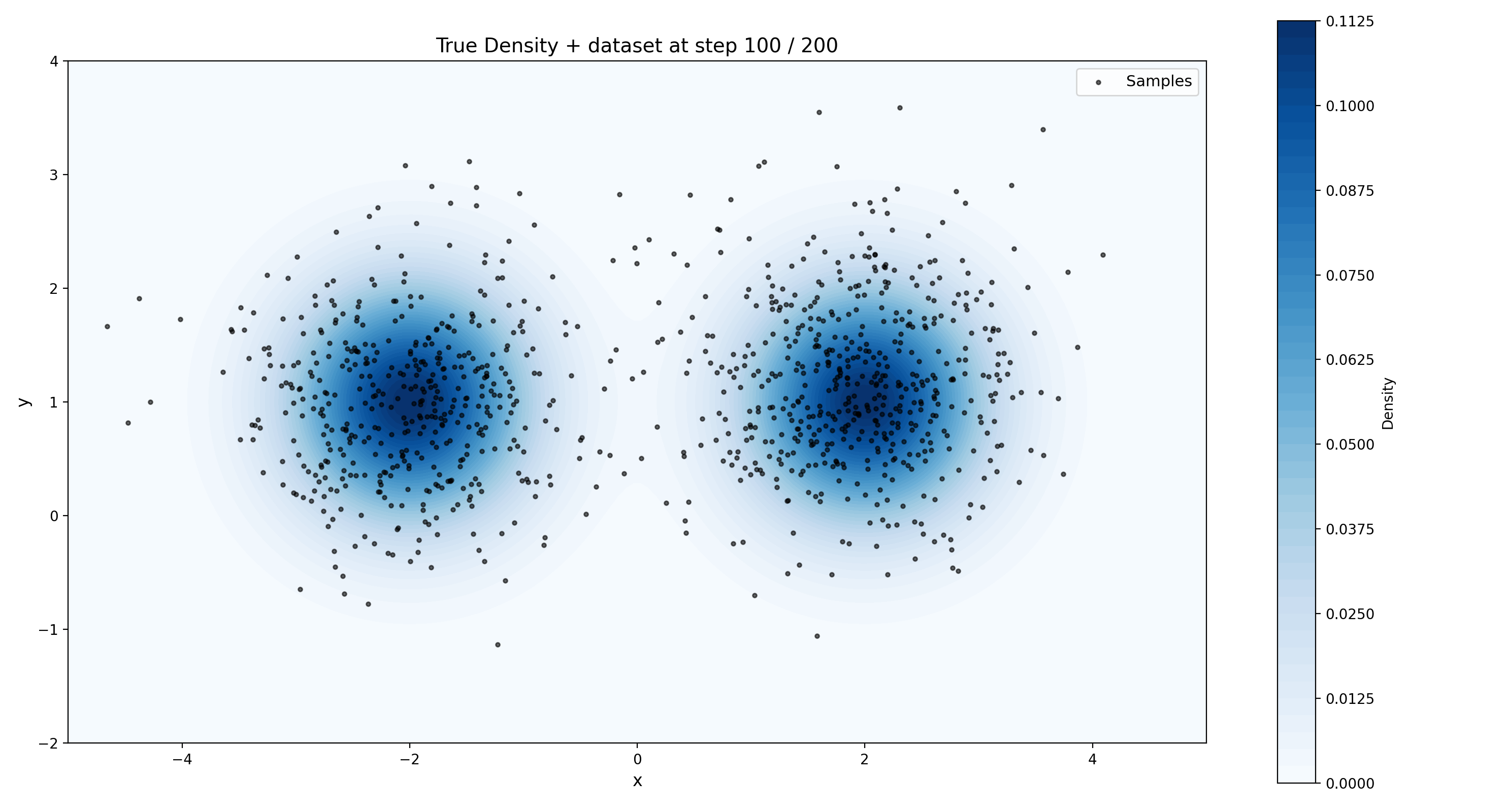}\hfill
    \includegraphics[width=0.48\textwidth]{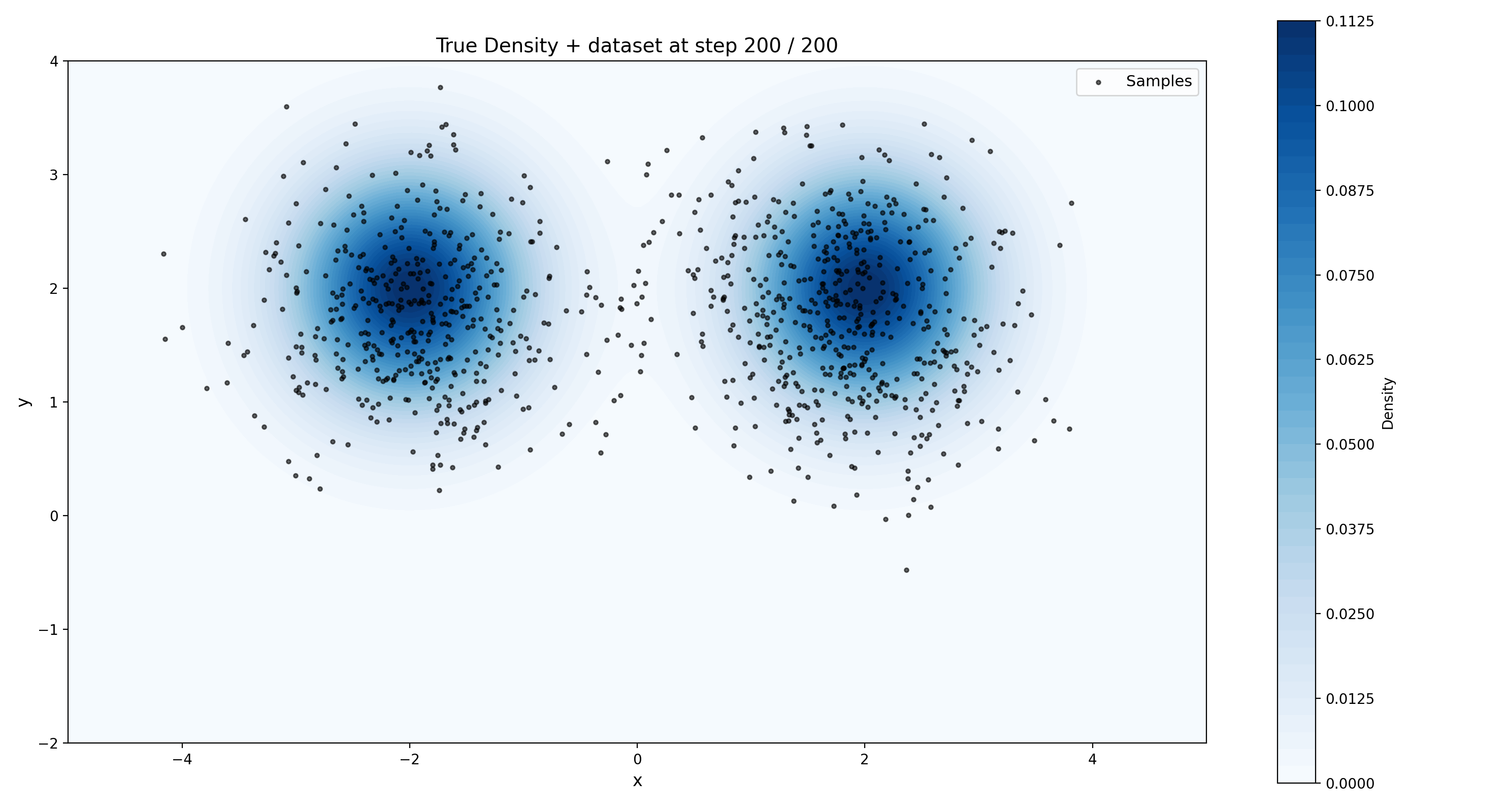}

    \caption{Iterative tilting on the 2-D GMM. Each row shows fixed-$N$ settings (top to bottom: $N=20,50,100,200$) at iterations $N/2$ (left) and $N$ (right). The right plot in each row corresponds to the target distribution.}
    \label{fig:2d_gaussian_mixture_samples}
\end{figure}

At each iteration of the tilting procedure, we know the closed form of the score of $p^{k}_0$ thanks to Section~\ref{subsec:gaussian_potential}, and we can thus compute the mean squared error between the learned score and the true score of the tilted distribution.
Thus, we monitor the RMSE (Root Mean Squared Error) between the learned score and the true score at each iteration $k=1,\ldots,N$:
\begin{equation}
    \mathrm{RMSE}_k = \E_{t\sim\mathrm{Unif}[0,1], X_t\sim p^{k}_t}\left[\|s_{\theta^k}(X_t,t) - \nabla_x \log p^{k}_t(X_t)\|^2\right]^{1/2}.
\end{equation}
We estimate this RMSE using $5,000$ samples from $p^{k}_t$ obtained with $200$ stochastic DDIM reverse steps.
We provide in Figure~\ref{fig:2d_gaussian_mixture_mse} the evolution of this RMSE during the tilting procedure for different values of $N$.

\begin{figure}
    \centering
    \includegraphics[width=0.48\textwidth]{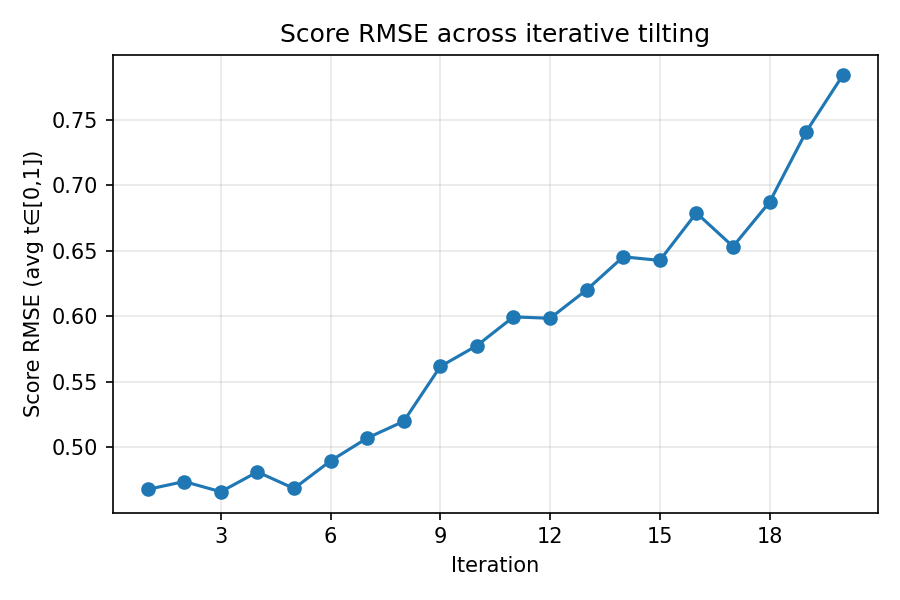}\hfill
    \includegraphics[width=0.48\textwidth]{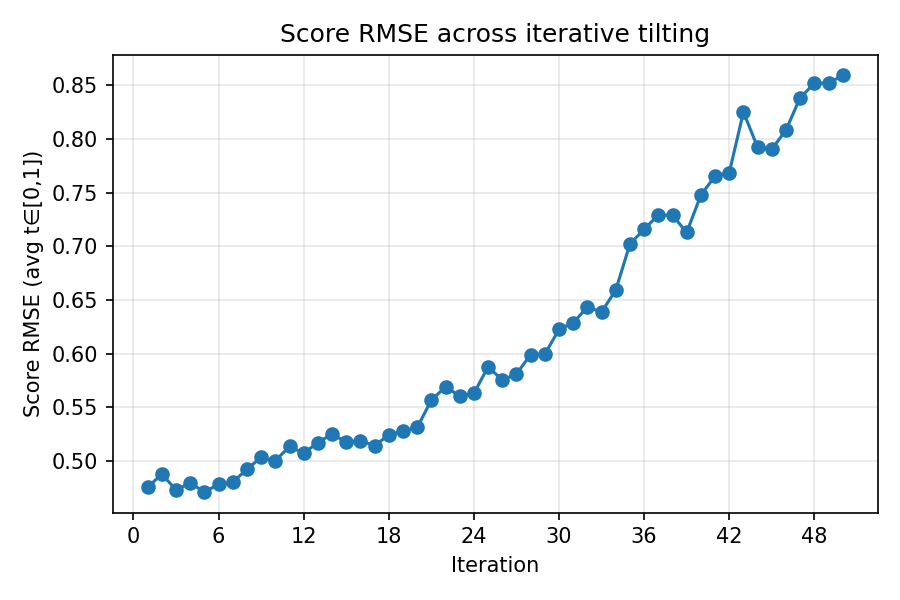}\\[0.6em]
    \includegraphics[width=0.48\textwidth]{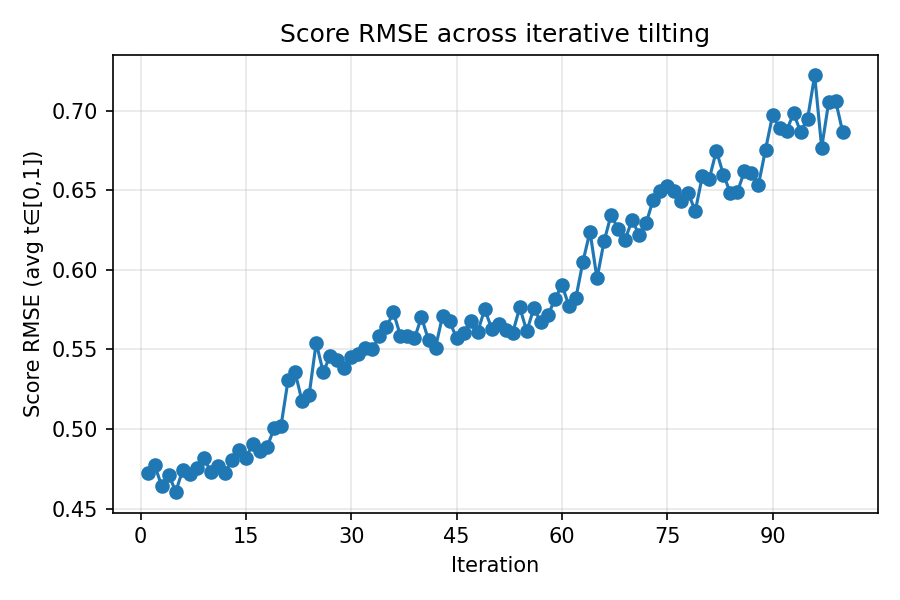}\hfill
    \includegraphics[width=0.48\textwidth]{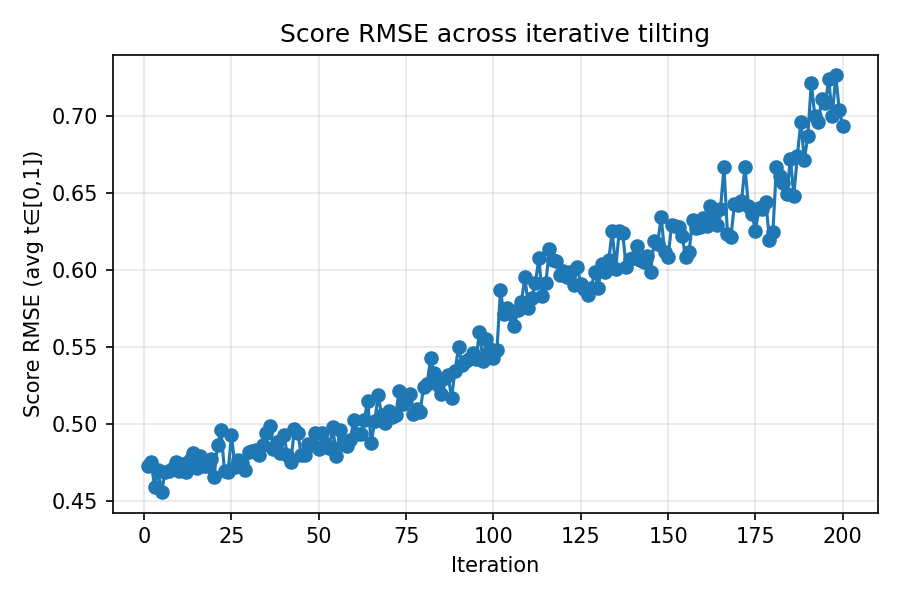}

    \caption{Evolution of the score error during iterative tilting for different $N$. We plot RMSE ($\sqrt{\mathrm{MSE}}$) versus tilt iteration for $N\in\{20,50,100,200\}$.}
    \label{fig:2d_gaussian_mixture_mse}
\end{figure}

 We report running times—total sampling time (s) and total training time (s)—and evaluation metrics (negative log-likelihood and mean squared mean error) in Table~\ref{tab:2d_gaussian_mixture_runtimes}.
\begin{table}[h]
    \centering
    \begin{tabular}{c|cc|cc}
        \toprule
        & \multicolumn{2}{c|}{Runtime (s)} & \multicolumn{2}{c}{Metrics} \\
        \cmidrule(lr){2-3}\cmidrule(lr){4-5}
        Number of tilts $N$ & Sampling (s) & Training (s) & Neg. log-lik. & MSE (mean) \\
        \midrule
        20  & 28.77 & 34.71 & 6.48 & 1.67 \\
        50  & 71.83 & 85.18 & 6.94 & 2.20 \\
        100 & 149.14 & 182.54 & 7.23 & 2.16 \\
        200 & 289.61 & 355.76 & 6.34 & 1.59 \\
        \bottomrule
    \end{tabular}
    \caption{Running times and metrics for iterative tilting on the 2-D GMM. Runtime columns report total sampling and training time in seconds. Metrics report negative log-likelihood and mean squared error on the mean.}
    \label{tab:2d_gaussian_mixture_runtimes}
\end{table}

These controlled experiments validate the mechanics of iterative tilting in a setting where the exact score is observable.
They also quantify the computational footprint of increasing the number of tilts, providing practical guidance for further studies on higher-dimensional datasets and richer reward models.

\section{Conclusion}
\label{sec:conclusion}

This paper introduced iterative tilting, a gradient-free method for fine-tuning diffusion models toward reward-tilted distributions. The core idea is to decompose a single large tilt $\exp(\lambda r)$ into $N$ sequential tilts $\exp(\lambda r / N)$, each producing a tractable score update. At iteration $k$, the score is approximated by
\begin{equation*}
    \nabla_x \log p^{k}_t(x) \approx \nabla_x \log p^{k-1}_t(x) + \frac{\lambda}{N}\, \mathrm{Cov}_{X_0\sim p^{k-1}_{0|t}(\cdot\mid x)}\!\left(\nabla_x \log q_{t|0}(x\mid X_0),\, r(X_0)\right),
\end{equation*}
valid for small $\lambda/N$ via first-order Taylor expansion. Crucially, the method requires only forward evaluations of $r$, avoiding the cost of backpropagating through sampling chains or solving adjoint equations.

We validated the approach on a two-dimensional Gaussian mixture with linear reward, where the exact tilted distribution is available in closed form. The experiments confirmed that the method successfully recovers the target distribution for $N\in\{20,50,100,200\}$, with score error (RMSE) converging at each tilt and runtime scaling linearly with $N$.

Compared to DRaFT~\citep{clark2024} and stochastic optimal control~\citep{domingo-enrich2024a}, which both require differentiating through the reward, iterative tilting trades a single gradient-based update for a sequence of gradient-free updates. This makes it well suited to settings where $r$ is a large pretrained model or is only available through black-box evaluations.

Several directions remain open: reducing variance in the single-sample covariance approximation, establishing principled criteria for selecting $N$, extending to high-dimensional domains with learned reward models, and combining with parameter-efficient methods such as LoRA~\citep{hu2021}.

\section*{Acknowledgments}
Jean Pachebat acknowledges support from the Chair Stress Test, Risk Management and Financial Steering, led by the French \'Ecole polytechnique and its foundation and sponsored by BNP Paribas.
Jean Pachebat also acknowledges access to the Cholesky HPC cluster operated by the IDCS unit at Institut Polytechnique de Paris.

\bibliography{bib/ftdiffusion}

\end{document}